\documentclass[letterpaper]{article} 
\usepackage{aaai2026}  
\usepackage{times}  
\usepackage{helvet}  
\usepackage{courier}  
\usepackage[hyphens]{url}  
\usepackage{graphicx} 
\usepackage{natbib}  
\usepackage{caption} 
\usepackage{algorithm}
\usepackage{algorithmic}
\usepackage{amssymb}
\usepackage{amsmath}
\usepackage{mathtools}
\usepackage{amsthm}
\usepackage{booktabs} 
\usepackage{multirow}
\usepackage{array} 
\usepackage{subcaption}
\usepackage{anyfontsize}
\usepackage{enumitem}
\usepackage{newfloat}
\usepackage{listings}
\DeclareCaptionStyle{ruled}{labelfont=normalfont,labelsep=colon,strut=off} 
\floatstyle{ruled}
\newfloat{listing}{tb}{lst}{}
\floatname{listing}{Listing}
\theoremstyle{plain}
\newtheorem{theorem}{Theorem}[section]

\theoremstyle{definition}

\theoremstyle{remark}

\renewcommand{\hat}{\widehat}

\newcommand{\Itrain}{\mathcal I_{\textrm{train}}}
\newcommand{\Itest}{\mathcal I_{\textrm{test}}}

\newcommand{\Ical}{\mathcal I_{\textrm{cal}}}

\title{\fontsize{14pt}{14pt}\selectfont Fast Conformal Prediction using Conditional Interquantile Intervals}
\author{
    Naixin Guo\textsuperscript{\rm 1}{$^*$}, Rui Luo\textsuperscript{\rm 2}\thanks{Corresponding author}, Zhixin Zhou \textsuperscript{\rm 3}{$^*$}\\
}
\affiliations{
    \textsuperscript{\rm 1}Department of Decision Analytics and Operations, City University of Hong Kong, Hong Kong SAR, China\\
    \textsuperscript{\rm 2}Department of Systems Engineering, City University of Hong Kong, Hong Kong SAR, China\\
    \textsuperscript{\rm 3}Alpha Benito Research, Los Angeles, USA\\

    naixinguo2-c@my.cityu.edu.hk, ruiluo@cityu.edu.hk, zzhou@alphabenito.com
}

\usepackage{bibentry}

\begin{document}

\maketitle

\begin{abstract}
We introduce Conformal Interquantile Regression (CIR), a conformal regression method that efficiently constructs near-minimal prediction intervals with guaranteed coverage. CIR leverages black-box machine learning models to estimate outcome distributions through interquantile ranges, transforming these estimates into compact prediction intervals while achieving approximate conditional coverage. We further propose CIR+ (Conditional Interquantile Regression with More Comparison), which enhances CIR by incorporating a width-based selection rule for interquantile intervals. This refinement yields narrower prediction intervals while maintaining comparable coverage, though at the cost of slightly increased computational time.
Both methods address key limitations of existing distributional conformal prediction approaches: they handle skewed distributions more effectively than Conformalized Quantile Regression, and they achieve substantially higher computational efficiency than Conformal Histogram Regression by eliminating the need for histogram construction. Extensive experiments on synthetic and real-world datasets demonstrate that our methods optimally balance predictive accuracy and computational efficiency compared to existing approaches.
\end{abstract}

\begin{links}
    \link{Code}{https://github.com/orince/CIR}
\end{links}

\section{Introduction}
\label{sec:introduction}
Conformal prediction is a powerful framework for constructing prediction intervals with finite-sample validity guarantees. By leveraging data exchangeability, conformal methods transform outputs from arbitrary machine learning algorithms into set-valued predictions that achieve the desired coverage level without imposing distributional assumptions on the underlying data.

Existing conformal regression methods primarily fall into two categories: directly predicting interval endpoints \cite{romano2019conformalized, kivaranovic2020adaptive,sesia2020comparison,gupta2022nested} or inverting estimated full conditional distributions \cite{izbicki2020flexible,chernozhukov2021distributional}. While effective in many cases, these approaches may produce suboptimal intervals if the data is skewed. Conformalized quantile regression (CQR) \cite{romano2019conformalized} may yield unbalanced intervals when the conditional distribution is skewed, whereas density-based methods can adapt to skewness but involve complex tuning and interpretation.
The conformal histogram regression (CHR) method \cite{sesia2021conformal} approximates the conditional distribution of $Y|X$ using histograms and seeks the shortest intervals with the desired coverage. It extracts more information from the estimated conditional distribution compared to other methods, such as CQR \cite{romano2019conformalized}. A key advantage of CHR is its ability to automatically adapt to the potential skewness of the data distribution, which sets it apart from methods that tend to produce symmetric intervals with fixed lower and upper miscoverage rates and may be suboptimal when dealing with data of unknown skewness \cite{romano2019conformalized, izbicki20a, chernozhukov2021distributional}.

However, CHR faces a potential limitation in terms of computational efficiency, particularly due to the process of discretizing the response space into bins, constructing histograms, and searching through them to locate interval boundaries. To overcome these inefficiencies, we propose Conformal Interquantile Regression (CIR), a novel method that directly uses the structure of conditional quantiles to determine intervals. CIR eliminates explicit response space binning by directly estimating the conditional probability of a new response falling into each interquantile interval, using a multi-output quantile regression model with equiprobable quantiles. This design enables CIR to adapt to data skewness while providing enhanced computational efficiency.

For each calibration sample, CIR begins with the narrowest interquantile interval and iteratively incorporates the next-shortest adjacent interval until the observed response $Y$ is covered. The number of intervals needed defines the conformity score. Intuitively, samples lying in long interquantile intervals receive larger scores, whereas those lying in shorter intervals receive smaller ones. Following the thresholding principle in \cite{luo2025threshold,romano2020classification}, we select a data-dependent cutoff on these scores so that the proportion of calibration responses contained in the corresponding merged intervals attains the target coverage level. This cutoff is then applied to test inputs to form the final prediction intervals.

To optimize the prediction interval further, we also propose Conditional Interquantile Regression with More Comparison (CIR+) as an extension of CIR. CIR+ incorporates an additional decision mechanism that evaluates whether to retain or discard specific interquantile intervals based on their length. Specifically, if the length of the interquantile interval containing $Y$ falls below a threshold determined during calibration, it is discarded. This step yields slightly narrower prediction set widths while maintaining comparable coverage performance.
Our analysis shows that prediction intervals obtained by thresholding the number of interquantile intervals in CIR and CIR+ guarantee marginal coverage.
Furthermore, these methods can achieve the desired conditional coverage and minimize expected prediction interval length asymptotically.
Experimental results on both synthetic and real-world datasets showcase the superiority of our proposed methods. CIR and CIR+ exhibit performance comparable to the state-of-the-art CHR method while significantly reducing computational demands. This combination of statistical efficiency and computational economy positions our approaches as powerful alternatives in the field of conditional prediction intervals.

The rest of this paper is structured as follows. Section~\ref{sec:related_work} discuss related work which provides context for our contributions. Section~\ref{sec:method} introduces the proposed CIR and CIR+ methods in detail. Following this, we present a brief theoretical analysis of our methods in Section~\ref{sec:theory}. Section~\ref{sec:experiment} presents numerical experiments comparing our methods with existing conformal regression techniques on simulated and real datasets.
Finally, we conclude the paper in Section~\ref{sec:conclusion}.

\section{Related Work}\label{sec:related_work}

Conformal Prediction (CP) has been successfully applied to both classification~\cite{luo2024trustworthy,luo2024entropy,luo2024weighted,wang2025trustworthy,zhangresidual} and regression tasks~\cite{luo2025threshold,luo2025volume,bao2025review}, demonstrating its flexibility across diverse real-world scenarios such as segmentation~\cite{luo2025conditional}, games~\cite{luo2024game, bao2025enhancing}, time-series forecasting~\cite{su2024adaptive}, and graph-based applications~\cite{luo2023anomalous, tang2025enhanced, luo2025conformalized, wang2025enhancing, luo2025conformal}.

Our work focuses on conformal regression and aims to produce compact, interpretable prediction intervals with guaranteed coverage. While \cite{izbicki2022cd} and \cite{luo2025threshold} (CTI) combine profile-distance–based similarity with multi-output quantile regression to approximate conditional densities, we instead prioritize interval-based uncertainty quantification, which—as emphasized by \cite{sesia2021conformal}—offers clearer interpretability and avoids irregular, overconfident regions. CTI employs a global length threshold on interquantile segments, often yielding highly fragmented prediction sets, whereas our CIR and CIR+ procedures maintain finite-sample marginal and conditional coverage while striving to produce a single contiguous interval whenever possible. Because CTI explicitly trades coverage for narrower sets, its behavior remains fundamentally distinct from ours even under unimodal conditional distributions and perfectly estimated quantiles.
\cite{gao2025volume} construct unconditional label–space level sets that do not depend on the conditional distribution, targeting density-based highest posterior density set rather than regression intervals.
However, their procedure requires repeatedly evaluating empirical coverage over all calibration-sample–induced breakpoints and solving a discrete global optimization over these candidate sets, which makes the method  computational expensive. 

\section{Proposed Method}
\label{sec:method}
\subsection{Problem Setup}
Consider a general regression problem with a dataset $\{(x_i, y_i)\}_{i=1}^n$, where $x_i \in \mathcal{X} \subseteq \mathbb{R}^d$ is the input feature vector and $y_i \in \mathcal{Y} \subseteq \mathbb{R}$ is the corresponding continuous response variable.  The dataset is split into three parts: a training set $\mathcal{D}_{\text{train}}$, a calibration set $\mathcal{D}_{\text{cal}}$, and a test set $\mathcal{D}_{\text{test}}$. The corresponding indices set are denoted by $\Itrain, \Ical$ and $\Itest$ respectively.  We assume that the examples in these sets are exchangeable.
Our goal is to find a prediction interval $C(X) \subseteq \mathcal{Y}$ for each test input $X$ such that the true response value $Y$ is included in $C(X)$ with a probability of at least $1-\alpha$, where $\alpha \in (0, 1)$ is a target significance level. Specifically, $C(X)$ should guarantee the finite-sample marginal coverage among all the samples in $\mathcal{D}_{\text{cal}}$ and $\mathcal{D}_{\text{test}}$:
\begin{equation*}
\mathbb{P}[Y \in C(X)] \geq 1-\alpha
\end{equation*}
for joint distribution for $X$ and $Y$. 

While achieving marginal coverage, we also aim for $C(X)$ to approximately achieve conditional coverage at level $1-\alpha$, where $\alpha \in (0, 1)$:
\begin{equation*}
\mathbb{P}[Y \in C(x)| X = x] \geq 1-\alpha
\end{equation*}
meaning that in practice the procedure should approximate this objective and, under appropriate conditions, achieve it asymptotically as the sample size tends to infinity.
Lastly, the prediction intervals are expected to be as narrow as possible. 

Imagine an {oracle} with access to $P_{Y \mid X}$, the distribution of $Y$ conditional on $X$, which leverages such information to construct optimal prediction intervals as follows. For simplicity, suppose $P_{Y \mid X}$ has a continuous density $f(y \mid x)$ with respect to the Lebesgue measure, although this could be relaxed with more involved notation. Then, the oracle interval for $Y \mid X=x$  would be:
\begin{align} \label{eq:oracle-interval}
  C_{\mathrm{oracle}}^{\alpha}(x) & = \left[ \ell_{1-\alpha}(x), \upsilon_{1-\alpha}(x) \right],
\end{align}
where, for any $\tau \in (0,1]$, $\ell_{\tau}(x)$ and $\upsilon_{\tau}(x)$ are defined as:
\begin{align} \label{eq:oracle-interval-int}
 \mathop{\mathrm{arg\,min}}_{(\ell, \upsilon) \in \mathbb{R}^2: \upsilon \geq \ell}  \left\{ |\upsilon-\ell| : \int_{\ell}^{\upsilon} f(y \mid x) dy \geq \tau \right\}.
\end{align}
This is the shortest interval with conditional coverage. Recall that CHR \cite{sesia2021conformal} replaces $f$ in~Equation~\eqref{eq:oracle-interval-int} with a histogram approximation, which is computation expensive. Specifically, if we partition the domain of $Y$ into $T$ equal parts  based on quantiles and obtain the corresponding quantiles $q_0(x), \ldots, q_T(x)$, where $q_t(x)$ is the $t/T$-quantile of the conditional distribution of $Y|X=x$, the interval derived from following expression approximates the one from Equation\eqref{eq:oracle-interval-int} when $T$ is sufficiently large:
\begin{equation} \label{eq:asym-oracle-interval}
 \begin{aligned}
    &\mathop{\mathrm{arg\,min}}_{
    \substack{l = 0, \dots,T-1 \\ l < u \leq T}
    } 
    \Bigl\{ |q_{u}(x) - q_{l}(x)|:\\&\quad \quad \quad\quad \quad
    \sum_{t=l}^{u-1} \mathbb{P}\bigl[Y \in (q_{t}(x),q_{t+1}(x)]\bigr]\geq \tau 
    \Bigr\}.
\end{aligned}   
\end{equation}

The corresponding interval of it would be:
\begin{align} \label{eq:asym-oracle-interval-int}
  C^{\alpha}(x) & = \left(q_{l_{\tau}}(x), \ q_{u_{\tau}}(x) \right].
\end{align}
This optimization over consecutive quantile-based bins motivates our interquantile-interval approach, which avoids explicit response-space binning.
\begin{algorithm}[tb]
\caption{Conditional Interquantile Regression (CIR)}
\label{alg:CIR}
\textbf{Input}:  Labeled data $\{(x_i, y_i)\}_{i\in\mathcal{I}}$, a data split ratio, unlabeled test data $\{x_i\}_{i\in\mathcal{I}_\text{test}}$, black-box learning algorithm $\mathcal{B}$, No. of interquantile intervals $T$, level $\alpha \in (0,1)$.
\begin{algorithmic}[1] 
\STATE  Randomly split the indices
  $\mathcal{I}$ into $\mathcal{I}_\text{train}$ and $\mathcal{I}_\text{cal}$.
\STATE  Train $\mathcal{B}$ on samples in $\mathcal{I}_\text{train}$, and obtain quantile estimation functions $\hat{q}_t$ for $t = 0,1,\ldots,T$.
\STATE For every $i \in \mathcal{I}_\text{cal} \cup \mathcal{I}_\text{test}$,
    evaluate $\hat{q}_t(x_i)$ for $t = 0,1,\ldots,T$.
\STATE For any given $k$ and  $i \in \mathcal{I}_\text{cal}$, define 
      $l_k$ as the lower endpoint of the shortest interval that encompasses $k$ consecutive interquantile intervals
    and the corresonding interval is defined as $C^\alpha_k(x_i) = (\hat{q}_{l_k}(x_i), \hat{q}_{l_k + k}(x_i)]$.  
\STATE  For $i \in \mathcal{I}_\text{cal}$, let $k_i$ be the smallest $k$ such that $y_i \in C^\alpha_k(x_i)$ (equivalent to $s(x_i, y_i)$ in Eq.\ref{eq:conformity-score-CIR}). 
 \STATE Compute $\hat{k}$ as the $r_{\alpha}$ smallest $k_i$, where $r_{\alpha} =\lceil (1 - \alpha)(1 +|\mathcal{I}_\text{cal}|) \rceil$ .
\end{algorithmic}
\textbf{Output}:  A prediction interval $C^\alpha_{\hat{k}}(x_i)$.
\end{algorithm}

\subsection{Conditional Interquantile Regression}

Our method directly selects consecutive interquantile intervals starting from the smallest ones, hence the name {\em conformal interquantile regression (CIR)}. Specifically,  we apply quantile regression on the training set $\mathcal{D}_{\text{train}}$ to predict the $t/T$-th quantile of the conditional distribution $Y|X=x$ for every $x\in \mathcal X$, where $t$ takes values from 0 to $T$ in increments of $1$. The estimated quantile for $t/T$ is denoted by $
\hat q_t(x)$ and
 the interquantile intervals are then defined as: 
\begin{align}\label{eq:interquantile:estimation}
    I_t(x) = (\hat q_{t-1}(x), \ \hat q_{t}(x) ]\quad\text{for}\quad t= \ 1,\dots, T.
\end{align}
Then Equation \ref{eq:asym-oracle-interval} can be approximately written as:
\begin{align} \label{eq:cir-intro}
 \mathop{\mathrm{arg\,min}}_
 {\substack{l = 0, \dots,T-1,\\ 1 \leq k \leq T-l-1}}  \left\{ \sum_{t = l}^{l+k}|I_{t+1}(x)|: \sum_{t=l}^{l+k} \mathbb{P}[Y \in I_{t+1}(x)]\geq \tau \right\},
\end{align}
where $l$ denotes the starting index of the interval and $k$ represents the number of consecutive interquantile intervals to be included, both determined by $\tau$. And the corresponding interval can be defined as 
\begin{align*}
  C^{\alpha}_{k(\tau)}(x) & = \left( \hat{q}_{l(\tau)} (x), \ \hat{q}_{l(\tau)+k(\tau)} (x) \right],
\end{align*}
Specifically, when $P_{Y \mid X}$ is unimodal, our estimated interval can be derived from
\begin{align} \label{eq:conformal-interval}
  C^{\alpha}_{\hat{k}}(x) & = \left( \hat{q}_{l_{\hat{k}}} (x), \ \hat{q}_{l_{\hat{k}}+\hat{k}} (x) \right],
\end{align}
where
\begin{equation}
\begin{aligned} \label{eq:k-l}
  \hat{k} \coloneqq  k(\hat{\tau});\ 
   l_{\hat{k}}  \coloneqq  l(\hat{\tau})  = \mathop{\mathrm{arg\,min}}_
 {l = 0, \dots,T-1}  \Big\{ \sum_{t = l}^{l+\hat{k}}|I_{t+1}(x)| \Big\}.
\end{aligned}
\end{equation}
Notice that $l_{\hat{k}}$ represents the lower endpoint index, with the subscript $\hat{k}$ reflecting the unimodal property that the number of included interquantile intervals uniquely determines both the narrowest interval and its corresponding lower endpoint. $\hat{\tau}$ will be determined by suitable conformity scores evaluated on the hold-out data, and it usually larger than  the true target coverage  $\tau$ if the model for $f$ is not very accurate to ensure this condition is met. However, if the quantiles are accurately estimated, ensuring that each interval has approximately the same probability, $1/T$, of covering the true label $Y$, and if the number of interquantile intervals approaches infinity, Equation~\eqref{eq:conformal-interval} will closely resemble that of the oracle in Equation~\eqref{eq:oracle-interval}.

The key quantity determining our prediction interval is $\hat{k}$, the number of interquantile intervals to include.
To determine $\hat{k}$, we first introduce a conformity score function:
\begin{align} \label{eq:conformity-score-CIR}
s(x,y) = \min\Bigl\{k \in {1,\dots,T}: y \in C^{\alpha}_k(x)\Bigr\}.
\end{align}
This function computes the minimum number of interquantile intervals needed to contain $y$.
The implementation follows an iterative process. Starting with the shortest interquantile interval, for each $i \in \mathcal{I}_\text{cal}$, if $y_i$ falls outside this interval, we expand our search to adjacent interquantile intervals (on either side, or one side at boundaries). We select the narrowest adjacent interval and combine it with our original interval. This process continues until $y_i$ falls within the expanded interval. The total number of interquantile intervals in this final interval defines $k_i$.

For computational efficiency, we implement the conformity score as $ s(x_i,y_i) = k_i$ for $i \in \mathcal{I}_\text{cal}$. We then define $r_{\alpha} =\lceil (1 - \alpha)(1 +|\mathcal{I}_\text{cal}|) \rceil$ and select
\begin{align} \label{eq:khat}
\hat{k} \leftarrow \text{the} \ \ r_{\alpha}\text{-th} \ \ \text{smallest} \ \ k_i,
\end{align}
ensuring $y_i \in C^\alpha_{\hat{k}}(x_i)$ holds for at least $r_{\alpha}$ instances in the calibration set. Finally, we use this $\hat{k}$ in Equation~\eqref{eq:conformal-interval} to obtain prediction sets for all $x_i, i \in \mathcal{I}_\text{test}$. The  procedure is detailed in Algorithm \ref{alg:CIR}.

\subsection{Conditional Interquantile Regression with More Comparison (CIR+)}

\begin{algorithm}[tb]
\caption{CIR with More Comparison (CIR+)}
\label{alg:algorithm2}
\textbf{Input:} Labeled data $\{(x_i, y_i)\}_{i\in\mathcal{I}}$, unlabeled test data $\{x_i\}_{i\in\mathcal{I}_\text{test}}$, a data split ratio, black-box learning algorithm $\mathcal{B}$,  No. of interquantile intervals $T$.
\begin{algorithmic}[1]
\STATE Randomly split the indices $\mathcal{I}$ into $\mathcal{I}_\text{train}$ and $\mathcal{I}_\text{cal}$
\STATE Train $\mathcal{B}$ on samples in $\mathcal{I}_\text{train}$, and obtain quantile estimation functions $\hat{q}_t$ for $t = 0,1,\ldots,T$.
\STATE For every $i \in \mathcal{I}_\text{cal} \cup \mathcal{I}_\text{test}$,
    evaluate $\hat{q}_t(x_i)$ for $t = 0,1,\ldots,T$.
\STATE  For any given $k$ and  $i \in \mathcal{I}_\text{cal}$, define 
      $l_k$ as the lower endpoint of the shortest interval that encompasses $k$ consecutive interquantile intervals
    and the corresonding interval is defined as $C^\alpha_k(x_i) = (\hat{q}_{l_k}(x_i), \hat{q}_{l_k + k}(x_i)]$.    
\STATE Define $s(x_i, y_i) =  k_i - 1 + e_{k_i}(x_i)$, where $k_i$ is the smallest $k$ such that $y_i \in C^\alpha_k(x_i)$ for  $i\in \mathcal{I}_\text{cal}$ and $e_{k_i}(x_i) \in (0,1)$ is a rescaled length of the $k_i$-th interquantile interval of $x_i$ as defined in Equation~\eqref{eq:e}.  
\STATE  Define $r_{\alpha} =\lceil (1 - \alpha)(1 +|\mathcal{I}_\text{cal}|) \rceil$ and then the threshold $\hat{s}$ is defined as the  $r_{\alpha}$th-smallest value of $s(x_i, y_i)$ for $i\in \mathcal{I}_\text{cal}$.
\STATE For $i \in \mathcal{I}_\text{test}$, $\hat{k}_i = \lfloor \hat{s} \rfloor +
\mathbf{1}\{e_{\lfloor \hat{s})\rfloor}(x_i) > \hat{s} - \lfloor \hat{s} \rfloor \}$.
\end{algorithmic}
\textbf{Output:} $C^\alpha_{{\hat{k}_i}}(x_i)$ for $y_i$, where $i \in \mathcal{I}_\text{test}$.
\end{algorithm}

In CIR, the conformity score is defined as the minimum number of interquantile intervals required to contain the response $y$. We then select the $r_{\alpha}$-th smallest number in the calibration set as the threshold for the minimum number of interquantile intervals in the prediction interval. However, when multiple intervals in the calibration set share the same width at the $r_{\alpha}$-th position, using only the interval count may reduce accuracy.
For example, suppose 5 samples in the CIR calibration process share rank $r_{\alpha}$, the selected $\hat{k}$ would span ranks $r_{\alpha}$ to $r_{\alpha}+4$.
To obtain a more precise ranking, we incorporate additional information and modify CIR into Conditional Interquantile Regression with More Comparison (CIR+). The conformity score is defined as:
\begin{equation}
    \label{eq:conformity-score-CIRFA}
\begin{aligned}
 s(x, y) =  \min \Bigl\{k - 1 + e_{k}(x): y \in C^{\alpha}_k(x)\Bigr\},
  \end{aligned}
\end{equation}
where 
\begin{equation}
\begin{aligned} \label{eq:e}
&e_k(x) = \frac{1}{c} \min \{\hat{q}_{l_{k-1} }(x)-\hat{q}_{l_{k-1}-1}(x),\\
& \quad \quad \quad \quad \quad \quad \quad \hat{q}_{l_{k-1} + {k}}(x)-\hat{q}_{l_{k-1} + {k} -1}(x) \},
\end{aligned}
\end{equation}
and $c$ can be any fixed constant which makes $e_k(x) \in (0,1)$.
In other words, when expanding our search from the shortest interquantile interval, we refine the process of incrementing $k$. Rather than simply increasing from $k-1$ to $k$ when we find an interquantile interval containing $y$, we add a rescaled length of the newly included interval—specifically, the adjacent interquantile interval with the smaller size. Typically, since the predicted quantiles for different samples are different, the lengths of the interquantile intervals are often different as well. This allows us to accurately select the $r_{\alpha}$-th smallest from these conformity scores.
In calibration step, we still first find $k_i$ as in CIR. Then, we only need to focus on the lengths of the interquantile intervals containing $y_i$ that correspond to the tied ranks at the $r_{\alpha}$-th smallest position. We then re-rank these tied scores according to Equation~\eqref{eq:conformity-score-CIRFA} and define the threshold as:
\begin{align} \label{eq:shat} 
 \hat{s} \leftarrow \text{the} \ \ r_{\alpha}\text{-th}  \ \  \text{smallest}  \ \ k_i - 1 + e_{k_i}(x_i).
\end{align}

In the prediction phase, we identify $C^\alpha_{\lfloor \hat{s} \rfloor}(x_i)$ for each sample in the test set. Subsequently, we determine whether to include the $\lfloor \hat{s} \rfloor$-th merged interquantile interval based on its length. Specifically, we compare the scaled length to $\hat{s} - \lfloor \hat{s} \rfloor$. If greater, we retain the segment and the prediction set becomes $C^\alpha_{\lceil\hat{s} \rceil}(x_i)$. If less, we discard it, leaving the set as $C^\alpha_{\lfloor \hat{s} \rfloor}(x_i)$. The procedure is shown in Algorithm \ref{alg:algorithm2}.

\begin{figure*}
    \centering
    \begin{subfigure}[t]{0.42\textwidth}
        \centering
        \includegraphics[width=\textwidth]{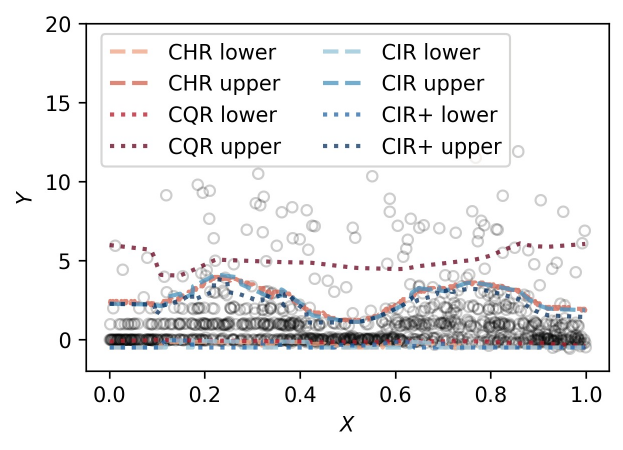}
        \label{fig:hr@1_wr}
    \end{subfigure}
    \hspace{0.001\textwidth}
    \begin{subfigure}[t]{0.46\textwidth}
        \centering
        \includegraphics[width=\textwidth]{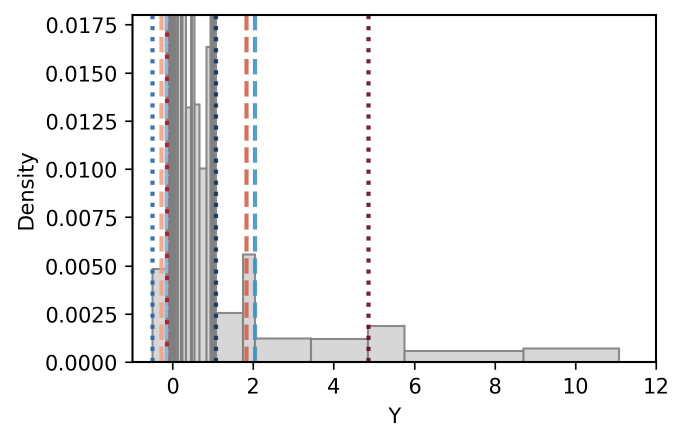}
        \label{fig:interquantile_wr}
    \end{subfigure}
    \caption{Comparison of CIR, CIR+, CQR \citep{romano2019conformalized}, and CHR \citep{sesia2021conformal} in a single-variable example. All methods use the same deep quantile model and guarantee 90\% marginal coverage. Left: Prediction bands as a function of $X$. Empirical marginal and estimated conditional coverage for all methods is 0.9 (except for CQR conditional coverage: 0.7). Average lengths: CIR (3.0), CIR+ (2.7), CHR (3.0), CQR (5.2). Right: Interquantile intervals at $X \approx 0.4$, using 50 quantiles with truncated high-density regions. Narrower intervals indicate higher densities. CIR and CIR+ select narrower intervals, with CIR+ choosing one fewer due to its scaled length being less than $\hat{s}-s$.}
    \label{fig:band_compare_wr}
\end{figure*}
\section{Theoretical Results}\label{sec:theory}
In this section, we mainly present the theoretical results for CIR and provide additional details, including the analysis of CIR+, in the appendix.
First, we show that CIR has the potential to achieve the optimal size for prediction intervals when
considering the marginal distribution if we assume that our quantile regression
model is sufficiently accurate.
The proof of  marginal coverage probability for CIR follows the same logic as the standard argument used in the general conformal prediction framework.
\begin{theorem}\label{thm:cir}(Finite-sample)
   If $(X_i,Y_i)$, for $i \in \mathcal{I}_\text{cal}\cup\mathcal{I}_\text{test} $, are exchangeable, then for $i \in \mathcal{I}_\text{test} $, the output of~Alg.~\ref{alg:CIR} satisfies:
\begin{align}
  \mathbb{P}\left[Y_{i} \in C^{\alpha}_{\hat{k}}(X_{i}) \right] \geq 1-\alpha.
\end{align} 
\end{theorem}
\begin{proof}
    Note that the score function $s(X_i, Y_i) = \min\Bigl\{k \in {1,\dots,T}: Y_i \in C^{\alpha}_k(X_i)\Bigr\}$ for $i\in \Ical\cup\Itest$ are also exchangeable. For any $(X_i,Y_i)$ in the test set, the rank of $s(X_i, Y_i)$ is smaller than $\hat{k}$ which in~Alg.~\ref{alg:CIR} with probability $\frac{\lceil (1+|\Ical|)(1-\alpha)\rceil}{1+|\Ical|}\ge 1-\alpha$.
\end{proof}
This theoretical result also holds for CIR+ (Alg.~\ref{alg:algorithm2}) and the detailed discussion is shown in the appendix.
As mentioned in Section~\ref{sec:method}, if the partition $T$ is large enough, the output in ~\eqref{eq:asym-oracle-interval}--\eqref{eq:asym-oracle-interval-int} is approximate to ~\eqref{eq:oracle-interval}--\eqref{eq:oracle-interval-int}.  The following result shows, if the quantiles are correctly estimated and the samples are i.i.d, the  $\hat{k}$ computed by CIR (Alg.~\ref{alg:CIR}) are asymptotically equal to the optimal as the sample size $|\mathcal{I}|\to \infty$ (also indicating $|\mathcal{I}_{cal}|\to \infty$).
Next theory relies on some assumptions :
\begin{enumerate}
\item \label{assump:iid} (i.i.d.) The samples are i.i.d., which is stronger than exchangeability.
\item (Consistency)\label{assump:consistency} When the black-box model estimates $P_{Y \mid X}$ consistently, each interquantile interval should contain the true label $Y$ with probability $1/T$, and the optimal number of interquantile intervals containing $Y$ should be $\lceil T(1-\alpha)\rceil$ to achieve $(1 - \alpha)$ coverage of true labels in the calibration set. This assumption is crucial, which similar to that in \cite{romano2019conformalized, sesia2021conformal}. While difficult in practice, it has been theoretically justified for some models like in \cite{meinshausen2006quantile}.

\item (Unimodality)\label{assump:unimodal} The true $P_{Y \mid X}$ is unimodal. 

\end{enumerate}
\begin{theorem}\label{thm:cir2}(Infinite-sample)
   Under Assumptions~\ref{assump:iid}--~\ref{assump:unimodal}, for $(X, Y)\in \mathcal{D}^{test}$, the $\hat{k}$ obtained from Alg.~\ref{alg:CIR} exhibits the following property:
\begin{align}
 &\mathbb{P}\left[ |\hat{k} - \lceil T(1-\alpha)\rceil|>\epsilon_n\right]  \leq \eta_n,\\
 &\mathbb{P}\left[\mathbb{P}\left[Y \in  C^{\alpha}_{\hat{k}}(X) \mid X \right] \geq 1-\alpha-\gamma_n\right] \geq 1-\zeta_n
\end{align} 
where we assume the calibration size $|\mathcal{I}_{cal}|= n \to \infty$ and $\epsilon_n \to 0, \eta_n \to 0, \gamma_n \to 0, \zeta_n \to 0$.
\end{theorem}

Intuitively, as $T$ large enough, the additional selection step in Alg. \ref{alg:algorithm2} becomes negligible, causing CIR+ to converge theoretically to CIR. The details of proof are shown in the appendix.
\paragraph{Remark:} 
The unimodality setting is not necessary for implementing  our methods or achieving marginal coverage in Thm. \ref{thm:cir} but only for properties of optimal interval length and conditional coverage in Thm. \ref{thm:cir2}.
In non-unimodal cases, the optimal prediction set may consist of multiple disjoint intervals rather than a single connected interval (guaranteed in the unimodal setting). To address this, a more effective approach involves ordering the interquantile intervals by length and making selections from this ordered set.
Nevertheless, empirical analysis in later sections and the appendix shows both methods maintain robust performance even on non-unimodal distributions, demonstrating their practical versatility.

\section{Experiments}\label{sec:experiment}

Our experimental implementation is based on the publicly available code for Conformal Histogram Regression (CHR) by \cite{sesia2021conformal}, which can be accessed at \texttt{https://github.com/msesia/chr}. We are deeply grateful for their work, serving as a significant source of inspiration for our work.
Our numerical experiments were conducted on Intel Xeon 2.10GHz CPUs in a computing cluster and each data set was analyzed using a single core.

\subsection{Synthetic Data Analysis} \label{sec:synthetic-data}

\begin{table}[t]
\centering
\setlength{\tabcolsep}{3pt} 
{\fontsize{7.2pt}{11.5pt}\selectfont
\begin{tabular}[t]{cccccc}
\toprule
$|\mathcal{I}|$ &  & CHR & CQR & CIR & CIR+ \\
\midrule
\multirow{4}{*}{500} & Marginal & 0.902 (0.001) & 0.901 (0.001) & 0.900 (0.002) & 0.897 (0.002) \\
 & Condit. & 0.881 (0.003) & 0.879 (0.004) & 0.878 (0.004) & 0.871 (0.004) \\
 & Width & 3.670 (0.039) & 5.165 (0.043) & 3.930 (0.057) & 3.790 (0.053) \\
 & Time & 36.597 (0.396) & 0.008 (0.000) & 0.072 (0.001) & 0.097 (0.001) \\
\midrule
\multirow{4}{*}{3000} & Marginal & 0.901 (0.001) & 0.900 (0.001) & 0.905 (0.001) & 0.900 (0.001) \\
 & Condit. & 0.887 (0.003) & 0.875 (0.003) & 0.895 (0.003) & 0.886 (0.003) \\
 & Width & 3.425 (0.017) & 5.320 (0.019) & 3.764 (0.028) & 3.632 (0.028) \\
 & Time & 73.735 (0.785) & 0.014 (0.000) & 0.404 (0.005) & 0.526 (0.007) \\
\midrule
\multirow{4}{*}{5000} & Marginal & 0.901 (0.001) & 0.900 (0.001) & 0.905 (0.001) & 0.900 (0.001) \\
 & Condit. & 0.889 (0.003) & 0.875 (0.003) & 0.895 (0.003) & 0.889 (0.003) \\
 & Width & 3.389 (0.013) & 5.345 (0.016) & 3.776 (0.022) & 3.643 (0.022) \\
 & Time & 103.637 (1.105) & 0.019 (0.000) & 0.663 (0.008) & 0.862 (0.010) \\
\bottomrule
\end{tabular}
}
\caption{Performance comparison of our methods (CIR and CIR+) and benchmarks on synthetic data for different sample size ($|\mathcal{I}|$), with skewness equals to $2.7$. The corresponding standard deviations are shown in parentheses. }
\label{tab:results_synthetic_n_wr}
\end{table}

\begin{table}[t]
\centering
\setlength{\tabcolsep}{3pt} 
{\fontsize{7.2pt}{11.5pt}\selectfont
\begin{tabular}[t]{cccccc}
\toprule
Skew &  & CHR & CQR & CIR & CIR+ \\
\midrule
\multirow{4}{*}{0.66} & Marginal & 0.901 (0.001) & 0.900 (0.001) & 0.906 (0.001) & 0.902 (0.001) \\
 & Condit. & 0.891 (0.005) & 0.883 (0.005) & 0.898 (0.004) & 0.889 (0.005) \\
 & Width & 5.627 (0.046) & 5.950 (0.031) & 5.900 (0.049) & 5.715 (0.054) \\
 & Time & 106.995 (2.049) & 0.018 (0.000) & 0.710 (0.015) & 0.911 (0.019) \\
\midrule
\multirow{4}{*}{1.71} & Marginal & 0.901 (0.001) & 0.900 (0.001) & 0.907 (0.001) & 0.902 (0.001) \\
 & Condit. & 0.891 (0.003) & 0.888 (0.003) & 0.897 (0.003) & 0.891 (0.003) \\
 & Width & 4.785 (0.025) & 5.436 (0.018) & 5.110 (0.026) & 4.870 (0.030) \\
 & Time & 103.026 (1.150) & 0.020 (0.000) & 0.687 (0.009) & 0.887 (0.011) \\
\midrule
\multirow{4}{*}{2.70} & Marginal & 0.901 (0.001) & 0.900 (0.001) & 0.905 (0.001) & 0.900 (0.001) \\
 & Condit. & 0.889 (0.003) & 0.875 (0.003) & 0.895 (0.003) & 0.889 (0.003) \\
 & Width & 3.389 (0.013) & 5.345 (0.016) & 3.776 (0.022) & 3.643 (0.022) \\
 & Time & 103.637 (1.105) & 0.019 (0.000) & 0.663 (0.008) & 0.862 (0.010) \\
\bottomrule
\end{tabular}
}
\caption{Performance comparison of our methods (CIR and CIR+) and benchmarks on synthetic data for different skewness of the conditional distribution of the response (Skew), with a sample size of 5000.}
\label{tab:results_synthetic_skew_wr}
\end{table}

We simulate a synthetic dataset with a one-dimensional feature $X$ and a continuous response $Y$ following \cite{romano2019conformalized,sesia2021conformal} and attach the details in the appendix.
Figure \ref{fig:band_compare_wr} illustrates the application of our methods to this toy data, comparing CIR and CIR+ with conformalized histogram regression (CHR) \cite{sesia2021conformal} and conformalized quantile regression (CQR) \cite{romano2019conformalized}. The left-hand side figure visualizes the resulting prediction bands for independent test data, comparing the analogous outputs across all methods. Both CIR and CIR+ are based on the same deep neural network and guarantee 90\% marginal coverage. CHR is capable of extracting information from all conditional quantiles estimated by the base model and automatically adapting to the estimated data distribution and produces relatively narrow intervals. Although our approaches omit the histogram step, it does not significantly impact our results. Both CIR and CIR+ yield intervals as narrow as CHR. In contrast, CQR can only leverage pre-specified lower and upper quantiles (e.g., 5\% and 95\%) and is therefore not adaptive to skewness. The right-hand side figure displays the interquantile intervals and their corresponding probabilities for $X \approx 0.4$. To elucidate the internal processes of the CIR and CIR+ algorithms, we have set only 50 quantiles for a tidier presentation. Additionally, we have truncated the regions with high density and focused on the areas where the information is concentrated. Note that the narrower interquantile intervals have higher densities. We can observe that both CIR and CIR+ select relatively narrower interquantile intervals at each step. In this sample, CIR+ selects one fewer interval compared to CIR because the scaled length of that interval is smaller than $\hat{s}-s$, leading to its exclusion. 
We simulate a synthetic dataset with a one-dimensional feature $X$ and a continuous response $Y$, drawn from the distribution illustrated in Figure~\ref{fig:band_compare_wr}. Our methods are applied to 2,000 independent observations from this distribution, utilizing the first 1,000 for training a deep quantile regression model and the remainder for calibration. Aligning with  \cite{sesia2021conformal}, we set 100 bins for CHR. 

Table~\ref{tab:results_synthetic_n_wr} shows CIR and CIR+ performance across sample sizes ($|\mathcal{I}|$) over 1000 independent experiments, with $|\mathcal{I}|$ split evenly between training and calibration and the number of interquantile intervals fixed at $T=100$.
 We evaluate marginal coverage, worst-slab conditional coverage as in \cite{cauchois2020knowing,romano2020classification}, average interval width, and computational cost per sample for calibration and prediction, with all methods using the same base model.
Both our methods and CHR achieve narrower intervals than CQR while maintaining marginal and conditional coverage. At smaller sample sizes, our method shows fluctuations due to imprecise quantile estimation from limited data, but as samples increase, quantile regression becomes more accurate and results stabilize.
CIR yields higher marginal coverage than the target level due to conformity score ties. In the ideal scenario, CIR+ should closely approach the coverage level when $N$ is large while CIR should show an additional approximately $1/(2T) = 0.005$ than coverage level $1 - \alpha = 0.9$ (As each interquantile interval carries $1/T$ probability, the threshold could occur anywhere, but on average falls at the midpoint). 
Notably, CHR is computationally intensive, with its runtime increasing as the sample size grows, which is primarily due to the increasingly time-consuming histogram construction process. As our confidence regions are computed directly from interquantile intervals and do not involve a binning process, the computational time required for calibrating and predicting each sample is negligible.

Table~\ref{tab:results_synthetic_skew_wr} shows results with 5000 samples under varying distributional skewness. Following \cite{sesia2021conformal}, we use a biased coin flip to invert $Y$ to $-Y$ for each data point, with coin bias as control. This generates distributions from skewed (Table~\ref{tab:results_synthetic_n_wr}) to symmetric $P_{Y \mid X}$. Results are shown against expected skewness $\mathbb{E}[(Y-\mu(X))^3/\sigma^3(X)]$, where $\mu(X)$ and $\sigma(X)$ are mean and standard deviation of $Y \mid X$.
For symmetric $P_{Y \mid X}$ (skewness near 0), all methods produce similar interval lengths. As skewness increases, our methods match CHR's performance across different asymmetry levels. While CHR yields narrowest intervals throughout, its high computational cost persists. CIR and CIR+ maintain efficiency regardless of skewness, balancing performance with practicality.
The appendix provides extended analyses, including comparisons with additional benchmarks such as distributional conformal prediction (DCP) (Chernozhukov et al., 2019) and DistSplit (Izbicki et al., 2020), along with evaluations across broader ranges of sample sizes, skewness levels, and results using random forest models.
\subsection{Real Data Analysis} \label{sec:real-data}
\begin{figure*}[t]
\centering
\includegraphics[width=0.8\textwidth]{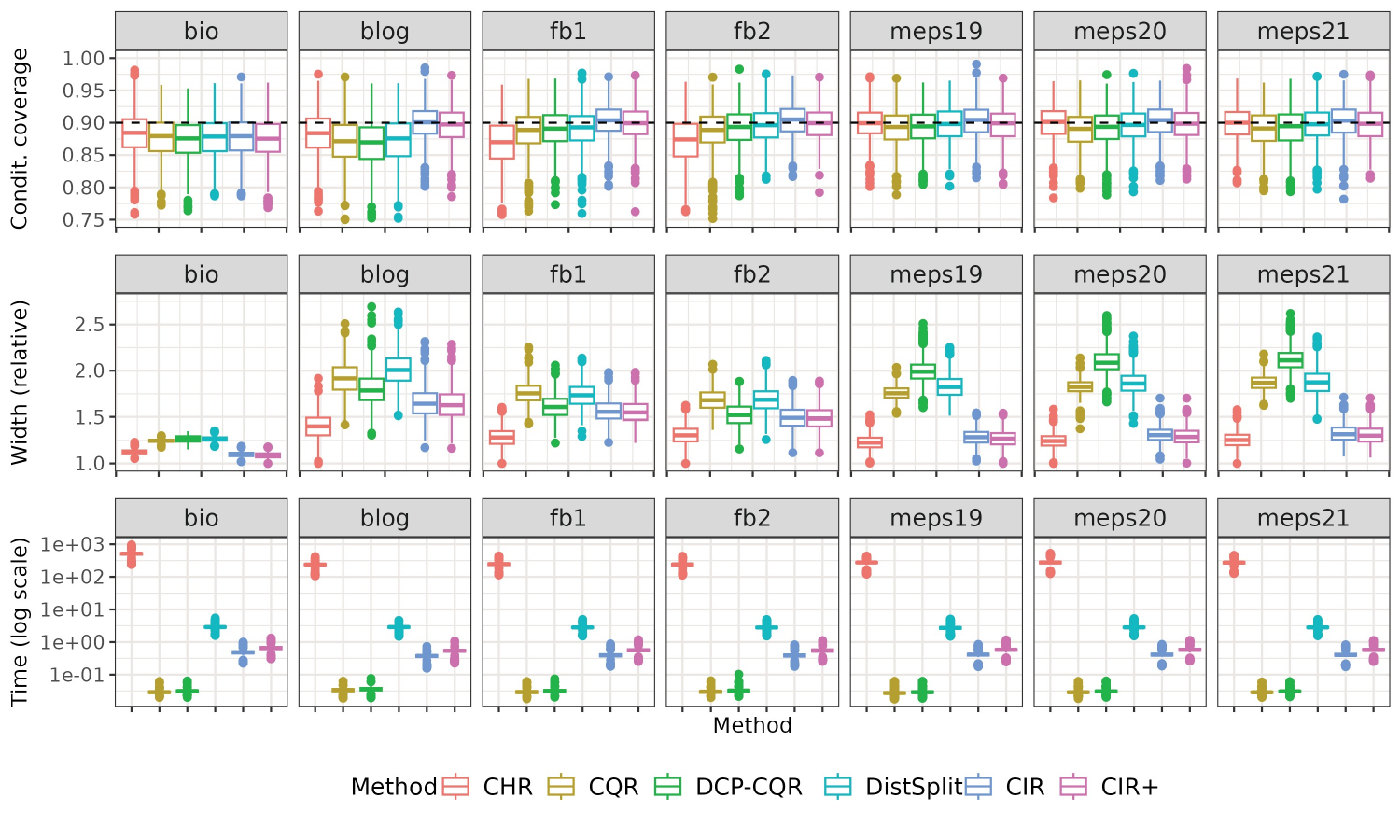} 
\caption{Performance of our method (CIR) compared to that of naive uncalibrated prediction intervals based on the same deep neural network regression model. Note that the top part of this plot shows marginal coverage.}
  \label{fig:results_real_nnet_naive}
\end{figure*}
We evaluate CIR and CIR+ on seven public-domain datasets previously analyzed in \cite{romano2019conformalized}: physicochemical properties of protein tertiary structure (bio) \cite{data-bio}, blog feedback (blog) \cite{blogData}, Facebook comment volume variants one (fb1) and two (fb2) \cite{facebookData}, and Medical Expenditure Panel Survey numbers 19 (meps19), 20 (meps20), and 21 (meps21) \cite{mepsData}. The first four datasets are from the UCI Machine Learning Repository \cite{Dua:2019}, while the MEPS datasets are from \cite{cohen2009medical}. For detailed descriptions of these datasets, we refer readers to \cite{romano2019conformalized}.
We compare our methods against CHR, CQR, DistSplit, and DCP-CQR. The latter, proposed in \cite{sesia2021conformal}, is designed to enhance the stability of DCP \cite{chernozhukov2019distributional} by integrating it with CQR \cite{romano2019conformalized}. All methods employ an identical deep quantile regression model for prediction. For comparison, we also present results using Random Forest method in the appendix. The Forests method \cite{meinshausen2006quantile} aggregates quantile estimates across an ensemble of decision trees, while the Neural Network approach \cite{taylor2000quantile} outputs multiple quantiles as a vector through shared parameters in a single network. Both methods are inherently designed to avoid the problem of crossing quantiles. Aligning with \cite{sesia2021conformal}, we set 1000 bins for CHR for this part. We assess their performance following the protocol described in the previous section, averaging results over 100 independent experiments per dataset. Each experiment utilizes 2000 samples for training, 2000 for calibration, and the remaining samples for testing. All features are standardized to zero mean and unit variance. The nominal coverage rate is set at 90\%.

Figure~\ref{fig:results_real_nnet_naive} compares conditional coverage and interval widths  across different methods and datasets. For meaningful cross-dataset comparison, prediction interval widths are normalized within each dataset, setting the smallest width to one. Computational times are presented on a logarithmic scale after similar normalization.
The results show all methods achieve satisfactory 
conditional coverage. While CHR and CIR+ alternately produce the most efficient (narrowest) prediction intervals across datasets, CIR+ and CIR exhibit substantially reduced computational overhead compared to CHR. This positions CIR+ as the optimal choice, effectively balancing computational efficiency with predictive performance.
CQR and DistSplit demonstrate comparable performance levels, whereas DCP-CQR tends to produce wider prediction intervals in some cases.
A comprehensive analysis, including marginal coverage (theoretically guaranteed for all methods) 
is provided in the appendix.

\section{Conclusion}\label{sec:conclusion}
In this paper, we introduce two novel conformal regression methods: Conformal Interquantile Regression (CIR) and its enhanced variant, CIR+. These methods are designed to construct prediction intervals that offer both guaranteed coverage and computational efficiency, effectively addressing key limitations of existing approaches. Specifically, our methods overcome the challenges of suboptimal performance on skewed data while avoiding the computational burden.
Our theoretical analysis demonstrates that both CIR and CIR+ not only guarantee marginal coverage but also achieve desired conditional coverage under appropriate conditions, while minimizing expected prediction interval length. 
Through extensive testing on synthetic and real datasets, we demonstrate that these methods match the performance of leading alternatives with substantially reduced computational cost. The two methods offer different solutions to the coverage-efficiency balance: CIR produces wider, more conservative intervals, while CIR+ generates narrower intervals by analyzing endpoint interquantile widths at a slight computational premium.  The performance of CIR and CIR+ can be affected with small datasets due to challenges in training quantile regression models effectively.  Our future work will focus on extending these frameworks to time series data, preserving its theoretical guarantees while expanding its applications.

\section*{Acknowledgments}
This work was supported by the National Natural Science Foundation of China (Grant 62506315) and City University of Hong Kong (Grants 9610639, 7020161).

\bibliography{aaai2026}

\appendix
\section{Additional Theoretical Results}
The proof of  marginal coverage probability for CIR follows the same logic as the standard argument used in the general conformal prediction framework.
\begin{theorem}\label{thm:CIR+}(Finite-sample)
   If $(X_i,Y_i)$, for $i \in \mathcal{I}_\text{cal}\cup\mathcal{I}_\text{test} $, are exchangeable, then for $i \in \mathcal{I}_\text{test} $, the output of Alg.~\ref{alg:algorithm2} satisfies:
\begin{align}
  \mathbb{P}\left[Y_{i} \in C^{\alpha}_{\hat{k_i}}(X_{i}) \right] \geq 1-\alpha.
\end{align} 
\end{theorem}
\begin{proof}
    Note that the score function $$s(X_i, Y_i) =  \min \Bigl\{k - 1 + e_k(X_i): Y_i \in C^{\alpha}_k(X_i)\Bigr\}$$ where 
\begin{equation*}
\begin{aligned} 
e_k(x) = \frac{1}{c} \min \{&\hat{q}_{l_{k-1} }(x)-\hat{q}_{l_{k-1}-1}(x),\\ &\hat{q}_{l_{k-1} + {k}}(x)-\hat{q}_{l_{k-1} + {k} -1}(x) \},
\end{aligned}
\end{equation*} for $i\in \Ical\cup\Itest$ are also exchangeable. For any $(X_i,Y_i)$ in the test set, the rank of $s(X_i, Y_i)$ is smaller than $\hat{s}$ which in Alg.~\ref{alg:algorithm2} with probability $\frac{\lceil (1+|\Ical|)(1-\alpha)\rceil}{1+|\Ical|}\ge 1-\alpha$.
\end{proof}

\paragraph{Proof of Theorem\ref{thm:cir2}}
Suppose these conditions hold:
\begin{enumerate}
\item (i.i.d.) The samples are i.i.d., which is stronger than exchangeability.
\item (consistency) The black-box model estimates $P_{Y \mid X}$ consistently, which means the quantiles are correctly estimated. When it holds, each interquantile interval should have probability $1/T$ of covering the true label $Y$ and the optimal number of interquantile intervals containing $Y$ should be $\lceil T(1-\alpha)\rceil$ in order to cover $(1 - \alpha)$  proportions of true labels in the calibration set.
\item (Unimodality) The true $P_{Y \mid X}$ is unimodal. 
\end{enumerate}
\renewcommand{\P}[1]{\mathbb{P}{\left[#1\right]}}
\newcommand{\Ps}{\mathbb{P}}
\newcommand{\E}[1]{\mathbb{E}{\left[#1\right]}}
Define
\begin{align*}
  \hat{F}(y \mid x) := {\hat{t}(y)}/T ,
\end{align*}
where $\hat{t}(y) = \max \{ t \in \{1,\ldots,T\} : y \leq \hat{q}_{t}\}$. Note that  $\hat{F}(\hat{q}_{t} \mid x)=t/T$.
Assume $|\mathcal{I}_{cal}| = n$, for all $t \in \{1,\ldots,T\}$,
\begin{align} \label{eq:consistency}
   \P{\E{\left( t/T - F(\hat{q}_{t} \mid X) \right)^2 \mid \mathcal{D}^{\mathrm{train}}} \leq \rho_n^{2}}
    & \geq 1-\rho_n^{2}.
\end{align}
Define the event $A_n$ as
\begin{align*}
  A_n := \left\{ x : \sup_{t \in \{1,\ldots,T\}} | \hat{F}(\hat{q}_{t} \mid x) - F(\hat{q}_{t} \mid x) | > \rho_n^{1/3} \right\}.
\end{align*}
 For any fixed $t \in \{1,\ldots,T\}$,
\begin{align*}
  &\P{X \in A_n}\\
   &= \P{\sup_{t' \in \{1,\ldots,T\}} | \hat{F}(\hat{q}_{t'} \mid x) - F(\hat{q}_{t'} \mid x) |^2 > \rho_n^{2/3}} \\
  & \leq T \, \P{| \hat{F}(\hat{q}_t \mid x) - F(\hat{q}_t \mid x) |^2 > \rho_n^{2/3} } \\
  & \leq T \left( \rho_n^{-2/3} \E{ \E{| \hat{F}(\hat{q}_t \mid x) - F(\hat{q}_t \mid x) |^2 \mid \mathcal{D}^{\mathrm{train}}} } \right) \\
  & \leq T \left( \rho_n^{-2/3} (4\rho_n^2+\rho_n^2) \right) \\
  & \leq 5T \rho_n^{4/3}
\end{align*}
The second inequality above is Markov's inequality. Let $T = \rho_n^{-1}$, we have $\P{X \in A_n^c} \geq  1 - 5\rho_n^{1/3}$.
Partition the calibration data points into
\begin{align*}
  & \mathcal{D}^{\mathrm{cal},a} := \{ i \in \mathcal{D}^{\mathrm{cal}} : X_i \in A_n\},\\
  & \mathcal{D}^{\mathrm{cal},b} := \{ i \in \mathcal{D}^{\mathrm{cal}} : X_i \in A_n^{\mathrm{c}}\},
\end{align*}
for any $\epsilon>0$,
\begin{align*}
  &\P{|\mathcal{D}^{\mathrm{cal},a}| \geq 5 n T \rho_n^{4/3} + \epsilon }\\
  & \leq \P{|\mathcal{D}^{\mathrm{cal},a}| \geq n \mathbb{P}\{X \in A_n\} + \epsilon } \\
  & \leq \P{\frac{1}{n} \sum_{i\in\mathcal{I}_{cal}} \mathbf{1}\{X_i \in A_n\}}  \geq \P{X_i \in A_n\} + \frac{\epsilon}{n} } \\
  & \leq \exp \left( -\frac{2\epsilon^2}{n} \right).
\end{align*}
Therefore, setting $\epsilon = c\sqrt{n \log n}$ for some constant $c>0$, yields 
\begin{align*}
  \P{|\mathcal{D}^{\mathrm{cal},a}| \geq 5 n T \rho_n^{4/3} + c\sqrt{n \log n} }
  & \leq n^{-2c^2}.
\end{align*}

For any fixed $t \in \{0,\ldots,T\}$, $k_i \geq t$, assume set $C_{k_i} = (\hat{q}_{\hat{l}_{k_i}-1},\hat{q}_{\hat{l}_{k_i}+k_i})$. In our setting, intervals containing more interquantile intervals are strictly longer than those containing fewer intervals—a property that naturally holds in unimodal settings. Formally, we assume $ C^{\alpha}_{k}(x) \subset C^{\alpha}_{k+1}(x)$ for each $k$.
Then, note that
\begin{align*}
  \P{k_i \leq t}
  & = \P{ Y_i \in C_{k_i} } \\
  & = F(\hat{q}_{\hat{l}_{k_i}+k_i}) - F(\hat{q}_{\hat{l}_{k_i}-1}) \\
  & \geq \hat{F}(\hat{q}_{\hat{l}_{k_i}+k_i}) - \hat{F}(\hat{q}_{\hat{l}_{k_i}-1}) - 2 \rho_n^{1/3} \\
  & \geq t/T - 2 \rho_n^{1/3}.
\end{align*}
Above, the first inequality follows from the definition of $\mathcal{D}^{\mathrm{cal},b}$.
Equivalently, we can rewrite this as
\begin{align*}
  \P{k_i > t + 2 \rho_n^{1/3} + \delta_n } \leq 1 - t /T - \delta_n,
\end{align*}
for any $\delta_n > 0$.
Now, partition $\mathcal{D}^{\mathrm{cal},b}$ into the following two disjoint subsets:
\begin{align*}
  \mathcal{D}^{\mathrm{cal}, b1} & := \{ i \in \mathcal{D}^{\mathrm{cal},b} : k_i \leq t + 2 \rho_n^{1/3} + \delta_n\}, \\
  \mathcal{D}^{\mathrm{cal}, b2} & :=\{ i \in \mathcal{D}^{\mathrm{cal},b} : k_i > t + 2 \rho_n^{1/3} + \delta_n\}.
\end{align*}
 Then we bound $|\mathcal{D}^{\mathrm{cal},b2}|$ with Hoeffding's inequality. 
For any $i \in \mathcal{D}^{\mathrm{cal}}$, define $\tilde{k}_i = k_i$ if $i \in \mathcal{D}^{\mathrm{cal},b}$ and $\tilde{k}_i  = t$ otherwise.
For any $\epsilon>0$,
\begin{align*}
  &\Ps\Big[|\mathcal{D}^{\mathrm{cal}, b2}|\geq n (1-t/T - \delta_n ) + \epsilon \Big] \\
  &\leq \Ps\Big[\frac{1}{n} \sum_{\mathcal{D}^{\mathrm{cal}, b}} \mathbf{1}\{\tilde{k}_i > t + 2 \rho_n^{1/3} + \delta_n\} \\  &\quad \quad \geq \Ps\Big[k_i > t + 2 \rho_n^{1/3} + \delta_n\Big] + \frac{\epsilon}{n} \Big] \\
  &\leq \Ps\Big[\frac{1}{n} \sum_{i=1}^{n} \mathbf{1}\{\tilde{k}_i > t + 2 \rho_n^{1/3} + \delta_n\}  \geq \\  &\quad \quad\Ps\Big[\tilde{k}_i > t + 2 \rho_n^{1/3} + \delta_n\Big] + \frac{\epsilon}{n} \Big] \\
  &\leq \exp \left( -\frac{2\epsilon^2}{n} \right).
\end{align*}
Therefore, setting $\epsilon = c\sqrt{n \log n}$, for some constant $c>0$, yields 
\begin{align*}
  \P{|\mathcal{D}^{\mathrm{cal}, b2}| \geq n (1-t/T - \delta_n) + c\sqrt{n \log n} } 
  & \leq n^{-2c^2}.
\end{align*}
As $|\mathcal{D}^{\mathrm{cal}, b1}| = n -|\mathcal{D}^{\mathrm{cal}, a}| - |\mathcal{D}^{\mathrm{cal}, b2}|$, combining the above result  yields:
\begin{align*}
  &\P{|\mathcal{D}^{\mathrm{cal}, b1}| \geq n t/T + n \delta_n - 5 n T \rho_n^{4/3} - 2 c \sqrt{n \log n}  } \\
  & \geq 1 - 2 n^{-2c^2}.
\end{align*}
If we choose $$\delta_n = t/(nT) +  T\left( \rho_n^{4/3} + \rho_n^2 \right)+ 2 c \sqrt{(\log n)/n},$$ this becomes
\begin{align*}
  &\P{|\mathcal{D}^{\mathrm{cal}, b1}| \geq t (n+1)/T  } \geq 1 - 2 n^{-2c^2},
\end{align*}
which means
\begin{align*}
  & \Ps[k(t) \leq t + t/(nT) + 2 \rho_n^{1/3} +  5 T \rho_n^{4/3} + 2 c \sqrt{(\log n)/n} ]
   \\&\geq 1 - 2 n^{-2c^2}
\end{align*}
As we note the $\lceil(1-\alpha)T \rceil\text{-th}$ $ k_i $ as $\hat{k}$, 
\begin{align*}
    \mathbb{P}[ \hat{k} \leq \lceil (1 - \alpha)T \rceil + &\lceil (1 - \alpha)T \rceil/nT + 2 \rho_n^{1/3} +  5 T \rho_n^{4/3} \\ 
   &+ 2 c \sqrt{(\log n)/n} ]\geq 1 - 2 n^{-2c^2}
\end{align*}
As $n\to \infty$, WLOG, let $T = \rho_n^{-1}$, $$\epsilon_{n} =1/n + 7 \rho_n^{1/3} +  2 c \sqrt{(\log n)/n}$$
and $\eta_n = 2 n^{-2c^2}$, which leads to our result.
Similarly as the proof before,
 \begin{align*}
  \P{k_i \leq t}
  & = \P{ Y_i \in C_{k_i} } \\
  & = F(\hat{q}_{\hat{l}_{k_i}+k_i}) - F(\hat{q}_{\hat{l}_{k_i}-1}) \\
  & \leq {F}(\hat{q}_{\hat{l}_{k_i}+k_i}) - {F}(\hat{q}_{\hat{l}_{k_i}-1}) + 2 \rho_n^{1/3} \\
  & \leq t/T + 2 \rho_n^{1/3}.
\end{align*}
Above, the first inequality follows from the definition of $\mathcal{D}^{\mathrm{cal},b}$. The second inequality follows from the observation that $C_{k_i}$ could not be optimal if $ F(\hat{q}_{\hat{l}_{k_i}+k_i}) - F(\hat{q}_{\hat{l}_{k_i}-1}) \geq t/T +2K\rho_n$ because it would be possible to obtain a shorter feasible interval by removing either the leftmost or the rightmost interquantile interval.

For any $\delta_n > 0$, partition $\mathcal{D}^{\mathrm{cal},b}$ into the following two disjoint subsets:
\begin{align*}
  \mathcal{D}^{\mathrm{cal}, b1} & := \{ i \in \mathcal{D}^{\mathrm{cal},b} : k_i \leq t - 2 \rho_n^{1/3} - \delta_n - 2K\rho_n\}, \\
  \mathcal{D}^{\mathrm{cal}, b2} & :=\{ i \in \mathcal{D}^{\mathrm{cal},b} : k_i > t - 2 \rho_n^{1/3} - \delta_n - 2K\rho_n\}.
\end{align*}
 Then we also bound $|\mathcal{D}^{\mathrm{cal},b2}|$ with Hoeffding's inequality. 
For any $i \in \mathcal{D}^{\mathrm{cal}}$, define $\tilde{k}_i = k_i$ if $i \in \mathcal{D}^{\mathrm{cal},b}$ and $\tilde{k}_i  = t$ otherwise.
For any $\epsilon>0$,
\begin{align*}
  &\Ps\Big[|\mathcal{D}^{\mathrm{cal}, b2}| \geq n (1-t/T - \delta_n ) + \epsilon \Big] \\
  &\leq \Ps\Big[\frac{1}{n} \sum_{\mathcal{D}^{\mathrm{cal}, b}} \mathbf{1}\{\tilde{k}_i \leq t - 2 \rho_n^{1/3} - \delta_n - 2K\rho_n\}  \\&\quad \quad \geq \Ps\Big[k_i \leq t - 2 \rho_n^{1/3} - \delta_n - 2K\rho_n\Big] + \frac{\epsilon}{n} \Big] \\
  &= \Ps\Big[\frac{1}{n} \sum_{i=1}^{n} \mathbf{1}\{\tilde{k}_i \leq t - 2 \rho_n^{1/3} - \delta_n - 2K\rho_n\} \\&\quad \quad \geq \Ps\Big[k_i \leq t - 2 \rho_n^{1/3} - \delta_n - 2K\rho_n\Big] + \frac{\epsilon}{n} \Big] \\
  &\leq \Ps\Big[\frac{1}{n} \sum_{i=1}^{n} \mathbf{1}\{\tilde{k}_i \leq t - 2 \rho_n^{1/3} - \delta_n - 2K\rho_n\}  \\&\quad \quad \geq \Ps\Big[\tilde{k}_i \leq t - 2 \rho_n^{1/3} - \delta_n - 2K\rho_n\Big] + \frac{\epsilon}{n} \Big] \\
  &\leq \exp \left( -\frac{2\epsilon^2}{n} \right).
\end{align*}
Therefore, setting $\epsilon = c\sqrt{n \log n}$, for some constant $c>0$, yields 
\begin{align*}
  \P{|\mathcal{D}^{\mathrm{cal}, b2}| \geq n (1-t/T - \delta_n) + c\sqrt{n \log n} } 
  & \leq n^{-2c^2}.
\end{align*}
As $|\mathcal{D}^{\mathrm{cal}, b1}| = n -|\mathcal{D}^{\mathrm{cal}, a}| - |\mathcal{D}^{\mathrm{cal}, b2}|$, combining the above result  yields:
\begin{align*}
  &\P{|\mathcal{D}^{\mathrm{cal}, b1}| \geq n t/T + n \delta_n - 5 n T \rho_n^{4/3} - 2 c \sqrt{n \log n}  } \\
  & \geq 1 - 2 n^{-2c^2}.
\end{align*}
If we choose $$\delta_n = t/(nT) +  T\left( \rho_n^{4/3} + \rho_n^2 \right)+ 2 c \sqrt{(\log n)/n},$$ this becomes
\begin{align*}
  \P{|\mathcal{D}^{\mathrm{cal}, b1}| \geq t (n+1)/T  } 
   \geq 1 - 2 n^{-2c^2},
\end{align*}
which means
\begin{align*}
  & \Ps\Big[ k(t) \geq  t - 2 \rho_n^{1/3} - 2K\rho_n -t/(nT) -  T\left( \rho_n^{4/3} - \rho_n^2 \right) \\
  & \quad \quad \quad \quad \quad \quad \quad \quad \quad \quad -  2 c \sqrt{(\log n)/n} \Big]
   \geq 1 - 2 n^{-2c^2}.
\end{align*}
As we note the $\lceil(1-\alpha)T \rceil\text{-th}$ $ k_i $ as $\hat{k}$, 
\begin{align*}
    &\mathbb{P}[ \hat{k} \geq \lceil (1 - \alpha)T \rceil - 2K\rho_n-\lceil (1 - \alpha)T \rceil/nT 
    \\
  &  \quad \quad - 2 \rho_n^{1/3} -  5 T \rho_n^{4/3} - 2 c \sqrt{(\log n)/n} ]\geq 1 - 2 n^{-2c^2}.
\end{align*}
Remember we proved before that $\P{X \in A_n^c} \geq  1 - 5\rho_n^{1/3}$ and the definition of $\hat{k}$ and $C_{\hat{k}}^\alpha$,
WLOG, let $T = \rho_n^{-1}$,
\begin{align*}
    &\mathbb{P}[\mathbb{P}[Y \in C_{\hat{k}}^\alpha(x) | X = x] \geq 1 - \alpha - 1/n - 7 \rho_n^{1/3} \\&
    \quad \quad \quad \quad \quad \quad-  2 K\rho_n- 2 c \sqrt{(\log n)/n} - 2\rho_n^{1/3}]  \\&\geq 1 - 5\rho_n^{1/3}- 2 n^{-2c^2}
\end{align*}
Therefore, $\gamma_n = 1/n - 9 \rho_n^{1/3} -  2 K\rho_n- 2 c \sqrt{(\log n)/n}$ and $\zeta_n = 2 n^{-2c^2}+ 5\rho_n^{1/3}$.
\section{Empirical Results}
We provides additional details about the experiments with synthetic and real data in this section.  When implementing Alg.~\ref{alg:CIR} and Alg.~\ref{alg:algorithm2}, we can initialize $k$ with $\lceil rate \times (1-\alpha)T\rceil$ intervals based on the theoretical optimal value $\hat{k}=\lceil(1-\alpha)T\rceil$. This strategy maintains result integrity while reducing computational overhead, with the $rate$ parameter tunable for larger datasets.
For fair comparison, we report results without rate optimization ($rate = 0$). Even in this unoptimized setting, both CIR and CIR+ demonstrate superior efficiency and interval quality compared to baseline methods.
\subsection{Expiriment Setting}
We estimate the distribution of $Y \mid X$ using the following quantile regression models and the hyperparameter set similar to \cite{sesia2021conformal}.
\begin{itemize}
    \item \textbf{Deep Neural Network (NNet):} 
        \begin{itemize}
            \item Architecture: Three fully connected layers with hidden dimension of 64
            \item Activation: ReLU functions
            \item Loss function: Pinball loss \cite{taylor2000quantile} for conditional quantile estimation
            \item Regularization: Dropout with rate 0.1
            \item Optimizer: Adam \cite{kingma2014adam} with learning rate 0.0005
            \item Training:
                \begin{itemize}
                    \item Maximum epochs: 2000
                    \item Optimal epoch count determined by cross-validation
                    \item Criterion: Minimization of loss function on hold-out data
                \end{itemize}
        \end{itemize}
    
    \item \textbf{Random Forest (RF):} 
        \begin{itemize}
            \item Implementation: Python \texttt{Scikit-garden} package for quantile regression forests \cite{meinshausen2006quantile}
            \item Hyperparameters:
                \begin{itemize}
                    \item Minimum samples to split an internal node: 50
                    \item Number of trees: 100
                    \item Other parameters: Default settings
                \end{itemize}
        \end{itemize}
\end{itemize}
\subsection{Synthetic Dataset}

We construct a synthetic regression dataset designed to capture nonlinear, heteroscedastic, and asymmetric noise patterns as \cite{romano2019conformalized,sesia2021conformal}. Let $X \in [0,1]$ be uniformly sampled inputs. For each input $x$, the response variable $Y$ is generated according to
\begin{equation}
    \begin{aligned}
     Y_0 = &\mathrm{Poisson}\!\left(\sin^{2}(2\pi x) + 0.1\right)
      + 0.2\, x\, \varepsilon_1\\
      &+ \mathbf{1}\{\varepsilon_2 < 0.09\}\left(5 + 2\varepsilon_3\right),   
    \end{aligned}
\end{equation}
where $\varepsilon_1,\varepsilon_2,\varepsilon_3$ are independent draws from $\mathcal{N}(0,1)$ or $\mathrm{Unif}(0,1)$ as appropriate. 
The first term introduces a smooth periodic structure with varying magnitude; the second term induces input-dependent variance; and the third term injects infrequent but large positive jumps, mimicking real-world bursty behaviors.  
To explicitly control the skewness of the observed data, a symmetry parameter $symmetry \in [0,1]$. The final response is given by
\[
Y =
\begin{cases}
-Y_0, & \text{with probability } symmetry,\\[3pt]
\;\;Y_0, & \text{otherwise}.
\end{cases}
\]
When $symmetry=0$, the distribution is strongly right-skewed due to the jump component; when $symmetry=0.5$, the distribution becomes approximately symmetric; and for $symmetry>0.5$, the distribution becomes left-skewed.

\section{Non-unimodal Model Case}
When working in a non-unimodal setting, we sort the interquantile bins by length from shortest to longest for each input and record each bin’s rank in this ordering. On the calibration set, each observation falls into exactly one bin; we take the rank of that bin as its conformity score. We then select a cutoff rank using the standard conformal calibration rule at the target miscoverage level. For a new input, the prediction set is the union of all bins with ranks at or below this calibrated cutoff.

\subsection{Additional Expiriments}

Due to space constraints, we present only a subset of our synthetic data results based on neural networks in the tables. For a more comprehensive view, Figure~\ref{fig:combined} displays additional results, while Figure~\ref{fig:combined_rf} provides comparisons using a random forest base model. In both experiments, our proposed methods and CHR consistently produce shorter prediction intervals.

We also evaluate these methods based on their worst-slab conditional coverage \cite{cauchois2020knowing}, estimated following the approach in \cite{romano2020classification}. Notably, all methods achieve the theoretically guaranteed 90\% marginal coverage.

Our analysis reveals that CHR, while effective, is computationally expensive in both the calibration and prediction phases compared to other methods. In contrast, CIR and CIR+ demonstrate optimal overall performance, striking a balance between computational efficiency and prediction accuracy. This balanced performance makes CIR and CIR+ the most advantageous methods when considering both time efficiency and predictive capabilities.

We also expand our analysis with additional results on real data. Figure~\ref{fig:results_real_rf_naive} presents further comparisons using a random forest base model. This figure is structured as follows:

\begin{itemize}
    \item The bottom panel illustrates the total time consumed by each method in the calibration and prediction steps. Notably, CHR demonstrates significantly higher computational demands compared to other methods.
    
    \item The middle panel displays the average interval length for each method across different datasets.
    
    \item The top panel compares these alternative methods in terms of their worst-slab conditional coverage \cite{cauchois2020knowing}, estimated following the approach in \cite{romano2020classification}.
\end{itemize}

While CIR and CIR+ produce slightly longer prediction intervals than CHR for the blog, fb1, and fb2 datasets, they exhibit superior conditional coverage performance. This balance between interval length and coverage accuracy, combined with their computational efficiency, positions CIR and CIR+ as the best overall performers when considering both time efficiency and predictive capabilities.

Importantly, all methods achieve the theoretically guaranteed 90\% marginal coverage. For a more comprehensive view of performance metrics, readers are directed to Table~\ref{tab:results_real}.
\begin{figure*}[t]
\centering
\begin{subfigure}{\textwidth}
    \centering
    \includegraphics[width=0.9\textwidth]{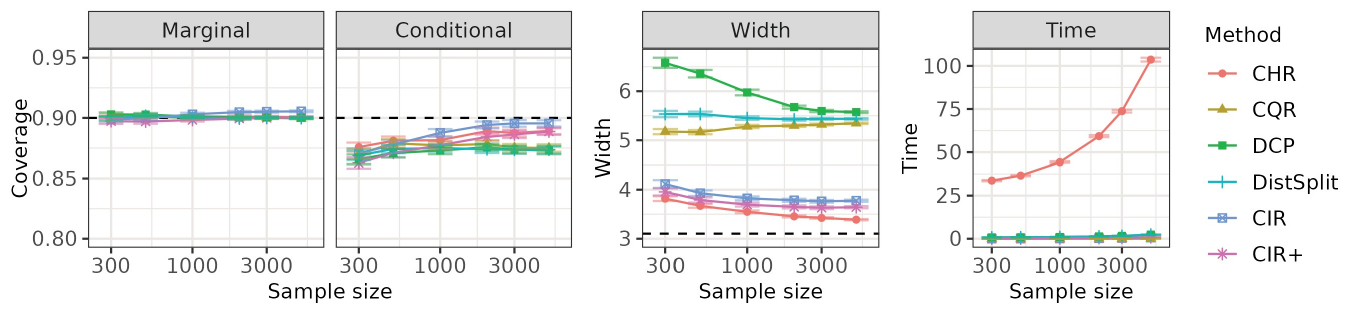}
\end{subfigure}
\begin{subfigure}{\textwidth}
    \centering
    \includegraphics[width=0.9\textwidth]{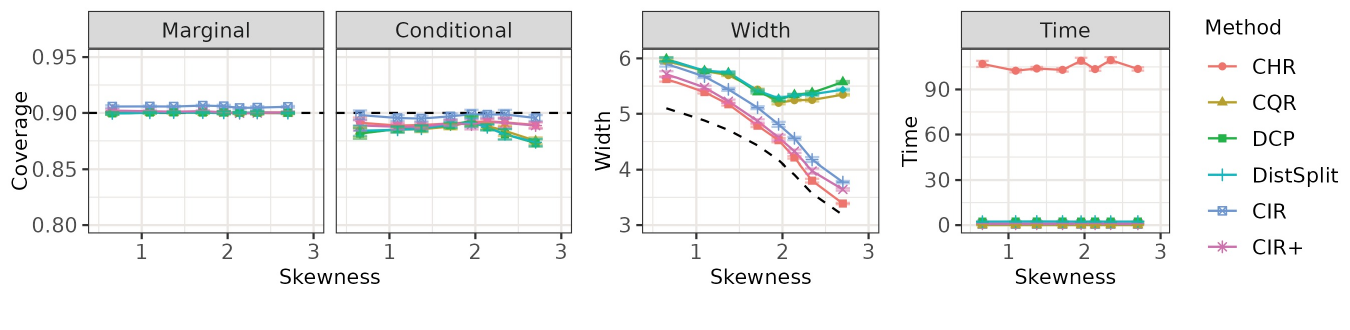}
\end{subfigure}

\caption{Performance comparison of our methods (CIR and CIR+) and benchmarks on synthetic data based on neural networks. Upper: Effects of sample size. The dashed lines and curves correspond to an omniscient oracle. The vertical error bars span two standard errors from the mean. Lower: Effects of distribution skewness of the conditional distribution of the response, with a sample size of 5000. The maximum skewness (near 3) corresponds to data in the upper figure.}
\label{fig:combined}
\end{figure*}
\begin{figure*}[t]
\centering
\begin{subfigure}{\textwidth}
    \centering
    \includegraphics[width=0.9\textwidth]{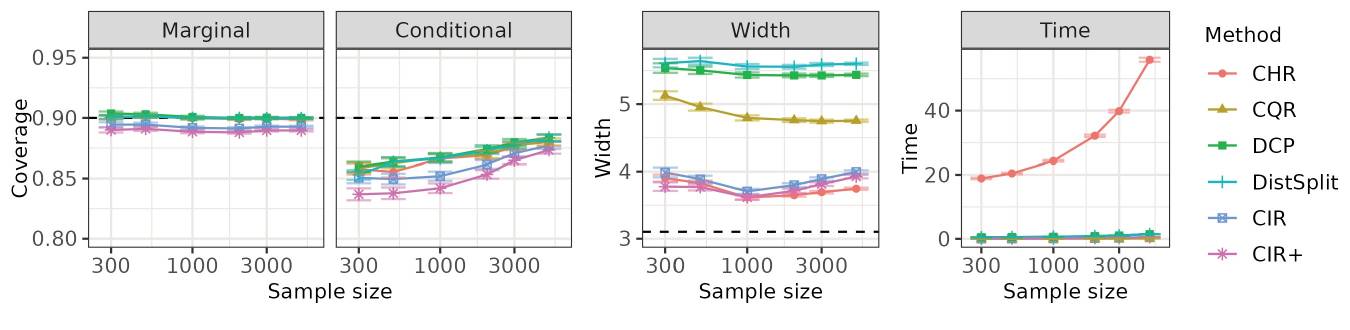}
\end{subfigure}
\begin{subfigure}{\textwidth}
    \centering
    \includegraphics[width=0.9\textwidth]{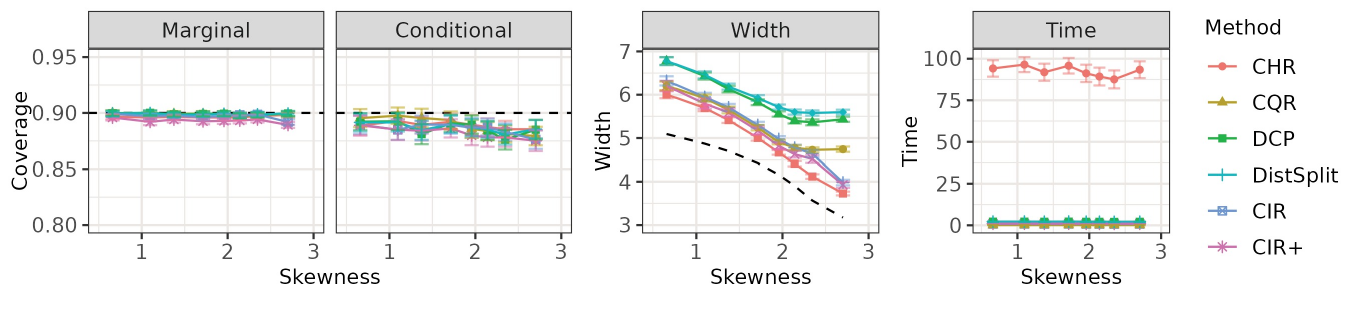}
\end{subfigure}

\caption{Performance comparison of our methods (CIR and CIR+) and benchmarks on synthetic data based on random forest. Upper: Effects of sample size. The dashed lines and curves correspond to an omniscient oracle. The vertical error bars span two standard errors from the mean. Lower: Effects of distribution skewness of the conditional distribution of the response, with a sample size of 5000. The maximum skewness (near 3) corresponds to data in the upper figure.}
\label{fig:combined_rf}
\end{figure*}

\begin{figure*}[t]
\centering
\includegraphics[width=0.95\textwidth]{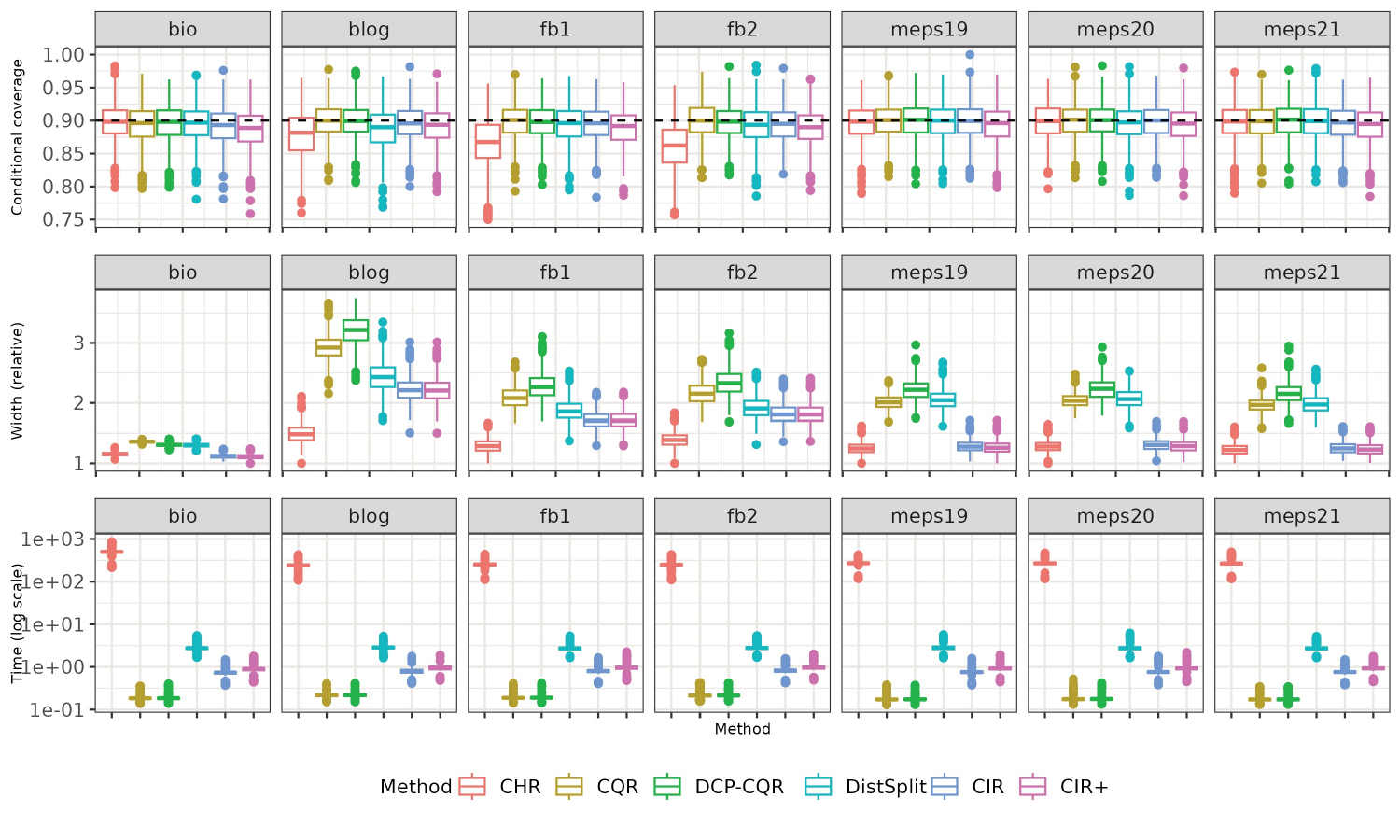} 
\caption{Performance of our methods compared to that of naive uncalibrated prediction intervals based on the  random forest regression model. Note that the top part of this plot shows marginal coverage.}
  \label{fig:results_real_rf_naive}
\end{figure*}

\begin{table*}[t]
\centering
\tiny

\begin{tabular}[t]{cccccccccc}
\toprule
\multicolumn{3}{c}{ } & \multicolumn{3}{c}{Neural Network} & \multicolumn{3}{c}{Random Forest} \\
\cmidrule(l{3pt}r{3pt}){3-6} \cmidrule(l{3pt}r{3pt}){7-10}
\multicolumn{2}{c}{ } & \multicolumn{2}{c}{Coverage} & \multicolumn{2}{c}{ } & \multicolumn{2}{c}{Coverage} & \multicolumn{1}{c}{ } \\
\cmidrule(l{3pt}r{3pt}){3-4} \cmidrule(l{3pt}r{3pt}){7-8}
Data & Method & Marginal & Condit. & Width & Time & Marginal & Condit. & Width & Time\\
\midrule
 & CHR & 0.90 (0.01) & 0.88 (0.03) & 13.1 (0.3) & 505.99 (97.50) & 0.90 (0.01) & 0.90 (0.03) & 10.7 (0.3) & 478.47 (106.51)\\

 & CQR & 0.90 (0.01) & 0.88 (0.03) & 14.5 (0.2) & 0.03 (0.00) & 0.90 (0.01) & 0.89 (0.03) & 12.6 (0.1) & 0.18 (0.03)\\

 & DCP & 0.90 (0.01) & 0.88 (0.03) & 14.6 (0.3) & 2.81 (0.42) & 0.90 (0.01) & 0.90 (0.03) & 11.9 (0.3) & 2.69 (0.47)\\

 & DCP-CQR & 0.90 (0.01) & 0.87 (0.03) & 14.7 (0.4) & 0.03 (0.00) & 0.90 (0.01) & 0.90 (0.03) & 12.1 (0.3) & 0.18 (0.03)\\

 & DistSplit & 0.90 (0.01) & 0.88 (0.03) & 14.7 (0.3) & 2.83 (0.43) & 0.90 (0.01) & 0.90 (0.03) & 12.0 (0.3) & 2.70 (0.48)\\

 & CIR & 0.90 (0.01) & 0.88 (0.03) & 12.7 (0.3) & 0.48 (0.09) & 0.90 (0.01) & 0.89 (0.03) & 10.3 (0.3) & 0.70 (0.14)\\

\multirow[t]{-7}{*}{\centering\arraybackslash bio} & CIR+ & 0.89 (0.01) & 0.87 (0.03) & 12.6 (0.4) & 0.65 (0.12) & 0.89 (0.01) & 0.89 (0.03) & 10.3 (0.4) & 0.85 (0.17)\\
\cmidrule{1-10}
 & CHR & 0.90 (0.01) & 0.88 (0.03) & 11.1 (1.1) & 229.03 (45.34) & 0.90 (0.01) & 0.88 (0.03) & 10.6 (1.1) & 231.77 (48.09)\\

 & CQR & 0.90 (0.01) & 0.87 (0.04) & 15.2 (1.5) & 0.03 (0.01) & 0.90 (0.01) & 0.90 (0.03) & 20.8 (1.5) & 0.21 (0.03)\\

 & DCP & 0.90 (0.01) & 0.88 (0.03) & 1422.3 (0.1) & 2.74 (0.47) & 0.90 (0.01) & 0.90 (0.03) & 1421.3 (0.1) & 2.76 (0.47)\\

 & DCP-CQR & 0.90 (0.01) & 0.87 (0.04) & 14.2 (1.5) & 0.04 (0.01) & 0.90 (0.01) & 0.90 (0.03) & 23.0 (2.0) & 0.21 (0.03)\\

 & DistSplit & 0.90 (0.01) & 0.87 (0.04) & 16.0 (1.6) & 2.74 (0.47) & 0.90 (0.01) & 0.89 (0.03) & 17.3 (1.7) & 2.78 (0.47)\\

 & CIR & 0.91 (0.01) & 0.90 (0.03) & 13.1 (1.3) & 0.36 (0.08) & 0.90 (0.01) & 0.90 (0.03) & 15.8 (1.4) & 0.77 (0.16)\\

\multirow[t]{-7}{*}{\centering\arraybackslash blog} & CIR+ & 0.90 (0.01) & 0.90 (0.03) & 12.9 (1.3) & 0.52 (0.11) & 0.89 (0.01) & 0.89 (0.03) & 15.7 (1.5) & 0.93 (0.19)\\
\cmidrule{1-10}
 & CHR & 0.90 (0.01) & 0.87 (0.04) & 10.6 (0.8) & 242.20 (45.94) & 0.90 (0.01) & 0.87 (0.04) & 11.5 (1.0) & 244.09 (52.70)\\

 & CQR & 0.90 (0.01) & 0.89 (0.03) & 14.6 (1.0) & 0.03 (0.00) & 0.90 (0.01) & 0.90 (0.03) & 18.7 (1.6) & 0.19 (0.03)\\

 & DCP & 0.90 (0.01) & 0.89 (0.03) & 1303.3 (0.1) & 2.77 (0.42) & 0.90 (0.01) & 0.90 (0.03) & 1302.5 (0.1) & 2.71 (0.49)\\

 & DCP-CQR & 0.90 (0.01) & 0.89 (0.03) & 13.3 (1.1) & 0.03 (0.00) & 0.90 (0.01) & 0.90 (0.03) & 20.3 (1.9) & 0.19 (0.03)\\

 & DistSplit & 0.90 (0.01) & 0.89 (0.03) & 14.4 (1.1) & 2.77 (0.42) & 0.90 (0.01) & 0.89 (0.03) & 16.7 (1.6) & 2.70 (0.49)\\

 & CIR & 0.90 (0.01) & 0.90 (0.02) & 12.9 (1.0) & 0.39 (0.08) & 0.90 (0.01) & 0.89 (0.03) & 15.3 (1.3) & 0.78 (0.16)\\

\multirow[t]{-7}{*}{\centering\arraybackslash fb1} & CIR+ & 0.90 (0.01) & 0.90 (0.03) & 12.9 (1.0) & 0.55 (0.11) & 0.89 (0.01) & 0.89 (0.03) & 15.4 (1.4) & 0.94 (0.20)\\
\cmidrule{1-10}
 & CHR & 0.90 (0.01) & 0.87 (0.03) & 11.0 (0.8) & 236.96 (43.77) & 0.90 (0.01) & 0.86 (0.04) & 11.1 (0.9) & 239.79 (51.39)\\

 & CQR & 0.90 (0.01) & 0.89 (0.03) & 14.2 (1.0) & 0.03 (0.00) & 0.90 (0.01) & 0.90 (0.03) & 17.1 (1.4) & 0.21 (0.03)\\

 & DCP & 0.90 (0.01) & 0.90 (0.03) & 1964.0 (0.1) & 2.76 (0.40) & 0.90 (0.01) & 0.90 (0.03) & 1963.3 (0.1) & 2.75 (0.47)\\

 & DCP-CQR & 0.90 (0.01) & 0.89 (0.03) & 12.8 (1.0) & 0.03 (0.01) & 0.90 (0.01) & 0.90 (0.03) & 18.6 (1.7) & 0.21 (0.03)\\

 & DistSplit & 0.90 (0.01) & 0.90 (0.03) & 14.2 (1.1) & 2.75 (0.39) & 0.90 (0.01) & 0.89 (0.03) & 15.2 (1.3) & 2.75 (0.46)\\

 & CIR & 0.90 (0.01) & 0.90 (0.03) & 12.5 (1.0) & 0.38 (0.07) & 0.90 (0.01) & 0.89 (0.03) & 14.4 (1.3) & 0.80 (0.15)\\

\multirow[t]{-7}{*}{\centering\arraybackslash fb2} & CIR+ & 0.90 (0.01) & 0.90 (0.03) & 12.5 (1.0) & 0.54 (0.10) & 0.89 (0.01) & 0.89 (0.03) & 14.4 (1.3) & 0.95 (0.19)\\
\cmidrule{1-10}
 & CHR & 0.90 (0.01) & 0.90 (0.03) & 20.3 (1.3) & 270.10 (51.03) & 0.90 (0.01) & 0.90 (0.03) & 19.4 (1.5) & 258.25 (54.26)\\

 & CQR & 0.90 (0.01) & 0.89 (0.03) & 29.0 (1.2) & 0.03 (0.00) & 0.90 (0.01) & 0.90 (0.03) & 31.2 (1.8) & 0.17 (0.03)\\

 & DCP & 0.90 (0.01) & 0.89 (0.03) & 559.3 (0.0) & 2.71 (0.48) & 0.90 (0.01) & 0.89 (0.03) & 559.0 (0.0) & 2.79 (0.54)\\

 & DCP-CQR & 0.90 (0.01) & 0.89 (0.03) & 33.0 (2.3) & 0.03 (0.00) & 0.90 (0.01) & 0.90 (0.03) & 34.3 (2.6) & 0.17 (0.03)\\

 & DistSplit & 0.90 (0.01) & 0.90 (0.03) & 30.2 (2.1) & 2.70 (0.47) & 0.90 (0.01) & 0.90 (0.03) & 31.8 (2.4) & 2.77 (0.52)\\

 & CIR & 0.90 (0.01) & 0.90 (0.03) & 21.2 (1.5) & 0.41 (0.09) & 0.90 (0.01) & 0.90 (0.03) & 19.9 (1.5) & 0.73 (0.16)\\

\multirow[t]{-7}{*}{\centering\arraybackslash meps19} & CIR+ & 0.90 (0.01) & 0.90 (0.03) & 21.0 (1.6) & 0.57 (0.12) & 0.90 (0.01) & 0.89 (0.03) & 19.6 (1.6) & 0.89 (0.19)\\
\cmidrule{1-10}
 & CHR & 0.90 (0.01) & 0.90 (0.03) & 19.2 (1.2) & 272.60 (58.48) & 0.90 (0.01) & 0.90 (0.03) & 18.6 (1.3) & 258.49 (57.63)\\

 & CQR & 0.90 (0.01) & 0.89 (0.03) & 28.1 (1.2) & 0.03 (0.00) & 0.90 (0.01) & 0.90 (0.03) & 29.6 (1.6) & 0.18 (0.03)\\

 & DCP & 0.90 (0.01) & 0.89 (0.03) & 520.3 (0.0) & 2.78 (0.47) & 0.90 (0.01) & 0.89 (0.03) & 520.1 (0.0) & 2.75 (0.59)\\

 & DCP-CQR & 0.90 (0.01) & 0.89 (0.03) & 32.4 (2.3) & 0.03 (0.00) & 0.90 (0.01) & 0.90 (0.02) & 32.3 (2.5) & 0.18 (0.03)\\

 & DistSplit & 0.90 (0.01) & 0.90 (0.03) & 28.7 (2.0) & 2.78 (0.45) & 0.90 (0.01) & 0.90 (0.03) & 30.0 (2.3) & 2.74 (0.60)\\

 & CIR & 0.91 (0.01) & 0.90 (0.03) & 20.2 (1.3) & 0.41 (0.09) & 0.90 (0.01) & 0.90 (0.03) & 18.9 (1.4) & 0.74 (0.18)\\

\multirow[t]{-7}{*}{\centering\arraybackslash meps20} & CIR+ & 0.90 (0.01) & 0.90 (0.02) & 19.9 (1.4) & 0.57 (0.12) & 0.90 (0.01) & 0.89 (0.03) & 18.7 (1.5) & 0.90 (0.21)\\
\cmidrule{1-10}
 & CHR & 0.90 (0.01) & 0.90 (0.03) & 20.4 (1.4) & 266.75 (50.50) & 0.90 (0.01) & 0.90 (0.03) & 20.0 (1.5) & 255.58 (56.01)\\

 & CQR & 0.90 (0.01) & 0.89 (0.03) & 30.3 (1.3) & 0.03 (0.00) & 0.90 (0.01) & 0.90 (0.03) & 32.1 (1.8) & 0.17 (0.03)\\

 & DCP & 0.90 (0.01) & 0.89 (0.03) & 531.3 (0.0) & 2.75 (0.43) & 0.90 (0.01) & 0.89 (0.03) & 531.0 (0.0) & 2.69 (0.46)\\

 & DCP-CQR & 0.90 (0.01) & 0.89 (0.03) & 34.3 (2.3) & 0.03 (0.00) & 0.90 (0.01) & 0.90 (0.03) & 35.1 (2.7) & 0.17 (0.03)\\

 & DistSplit & 0.90 (0.01) & 0.90 (0.03) & 30.4 (2.2) & 2.76 (0.43) & 0.90 (0.01) & 0.90 (0.03) & 32.3 (2.5) & 2.69 (0.47)\\

 & CIR & 0.91 (0.01) & 0.90 (0.03) & 21.5 (1.6) & 0.40 (0.08) & 0.90 (0.01) & 0.90 (0.03) & 20.4 (1.6) & 0.73 (0.14)\\

\multirow[t]{-7}{*}{\centering\arraybackslash meps21} & CIR+ & 0.90 (0.01) & 0.90 (0.03) & 21.2 (1.7) & 0.56 (0.11) & 0.89 (0.01) & 0.89 (0.03) & 20.1 (1.7) & 0.89 (0.18)\\
\bottomrule
\end{tabular}

\caption{Performance metrics for our proposed methods and established benchmarks across multiple real datasets, utilizing either a deep neural network or random forest as the base model. Values represent averages from 1000 random test sets, with standard deviations in parentheses.}
\label{tab:results_real}
\end{table*}
\end{document}